\documentclass{llncs}

\usepackage{llncsdoc}
\usepackage{graphicx}
\usepackage{amsmath,amssymb}

\usepackage{graphicx,type1cm,eso-pic,color}

\newcommand{\argmin}[1]{\underset{#1}{\operatorname{argmin}}\;}
\newcommand{\mR}{\mathbb{R}}
\newcommand{\mE}{\mathbb{E}}

\begin{document}

\title{Large Margin Multiclass Gaussian Classification with Differential Privacy}
\author{Manas A. Pathak and Bhiksha Raj}

\institute{
  Carnegie Mellon University \\ %, Pittsburgh, PA 15213, USA \\
  \texttt{\{manasp, bhiksha\}@cs.cmu.edu}
}

\maketitle

\begin{abstract}
As increasing amounts of sensitive personal information is aggregated
into data repositories, it has become important to develop mechanisms
for processing the data without revealing information about individual
data instances. The differential privacy model provides a framework
for the development and theoretical analysis of such mechanisms. In this
paper, we propose an algorithm for learning a discriminatively trained
multi-class Gaussian classifier that satisfies differential privacy
using a large margin loss function with a perturbed regularization
term. We present a theoretical upper bound on the excess risk of the
classifier introduced by the perturbation.

\end{abstract}

\section{Introduction}
In recent years, vast amounts of personal data is being aggregated in
the form of medical, financial records, social networks, and
government census data. As these often contain sensitive information,
a database curator interested in releasing a function such as a
statistic evaluated over the data is faced with the prospect that it
may lead to a breach of privacy of the individuals who contributed to
the database.
It is therefore important to develop techniques for
retrieving desired information from a dataset without revealing any
information about individual data instances. \emph{Differential
privacy}~\cite{Dwork06} is a theoretical model proposed to
address this issue. A query mechanism evaluated over a dataset is
said to satisfy differential privacy if it is likely to produce the
same output on a dataset differing by at most one element. This
implies that an adversary having complete knowledge of all data
instances but one along with {\em a priori} information about the
remaining instance, is not likely to be able to infer any more
information about the remaining instance by observing the output of the
mechanism.

One of the most common applications for such large data sets such as the ones
mentioned above is for training classifiers that can be used to categorize
new data. If the training data contains private data instances, an
adversary should not be able to learn anything about the individual
training dataset instances by analyzing the output of the
classifier. Recently, mechanisms for learning differentially private
classifiers have been proposed for logistic
regression~\cite{ChaudhuriM08}. In this method, the objective function
which is minimized by the classification algorithm is modified by
adding a linear perturbation term. Compared to the original
classifier, there is an additional error introduced by the
perturbation term in the differentially private classifier. 
It is important to have an upper bound on this error as a cost of
preserving privacy.

The work mentioned above is largely restricted to binary
classification, while multi-class classifiers are more useful in many
practical situations. In this paper, we propose an algorithm for
learning multi-class Gaussian classifiers which satisfies differential
privacy. Gaussian classifiers that model the distributions of
individual classes as being generated from Gaussian distribution or a
mixture of Gaussian distributions~\cite{mclachlan2000} are 
commonly used as multi-class classifiers. We use a large margin
discriminative algorithm for training the classifier introduced by Sha
and Saul~\cite{ShaS06}. To ensure that the learned multi-class
classifier preserves differential privacy, we modify the objective
function by introducing a perturbed regularization term.
%% To summarize, the contributions of the paper are:
%% \begin{itemize}
%% \item An algorithm for multi-class classification satisfying differential privacy.
%% \item A new technique to create differentially private classifiers --
%%   perturbing the regularization term.
%% \item Proof of differential privacy and a bound on the excess risk
%%   $O\left(\frac{C}{d\epsilon^2}\right)$ which is significantly smaller
%%   than the state of art and is proportional to the number of classes.
%% \end{itemize}

\section{Differential Privacy}
In recent years, the differential privacy model proposed by Dwork,
\emph{et al.}~\cite{Dwork06} has emerged as a robust standard for data
privacy. It originated from the statistical database model, where the
dataset $D$ is a collection of elements and a randomized \emph{query
  mechanism} $M$ produces a response when performed on a given
dataset. Two datasets $D$ and $D'$ differing by at most one element
are said to be \emph{adjacent}. There are two proposed
definitions for adjacent datasets one based on symmetric difference --
$D'$ containing of one entry less than $D$, and one
based on substitution -- one entry of $D'$ differs in
value from $D$. We use the substitution definition of adjacency
previously used by \cite{KasiviswanathanLNRS08,ChaudhuriM08}, where
the one entry of the dataset $D =\{x_1,\ldots,x_{n-1},x_n\}$ is
modified to result in an adjacent dataset
$D'=\{x_1,\ldots,x_{n-1},x_n'\}$.  The query mechanism
$M$ is said to satisfy differential privacy if the probability of $M$
resulting in a solution $S$ when performed on a dataset $D$ is very
close to the probability of $M$ resulting in the same solution $S$
when executed on an adjacent dataset $D'$. Assuming the query
mechanism to be a function $M : D \mapsto \text{range}(M)$ with a probability
function $P$ defined over the space of $M$, differential privacy is
formally defined as follows.
\begin{definition}
A randomized function $M$ satisfies $\epsilon$-differential privacy if
for all adjacent datasets $D$ and $D'$ and for any $S \in \text{range}(M)$,
\[ \left|\log\frac{P\left(M(D) = S\right)}{P\left(M(D') = S\right)}\right|
\le \epsilon. \]
\end{definition}
The value of the $\epsilon$ parameter, which is referred to as
\emph{leakage}, determines the degree of privacy. As there is always a
trade-off between privacy and utility, the choice of $\epsilon$ is
motivated by the requirements of the application.

In a machine learning setting, the query mechanism can be thought of
as an algorithm learning the classification, regression or density
estimation rule which is evaluated over the training dataset. The
output of an algorithm satisfying differential privacy is likely to be
same when the value of any single dataset instance is modified, and
therefore, no additional information can be obtained about any
individual training data instances with certainty by observing the
output of the learning algorithm, beyond what is already known to an
adversary.  Differential privacy is a strong definition of privacy --
it provides \emph{ad omnia} guarantee as opposed to most other models
that provide \emph{ad hoc} guarantees against specific set of attacks
and adversarial behaviors.

%% Dwork, \emph{et al.}~\cite{DworkMNS06} show
%% that for a differentially private mechanism, an adversary who observes
%% the values for all except one entry in the dataset and has prior
%% information about the last entry cannot learn anything with high
%% certainty about the value of the last entry beyond what was known
%% \emph{a priori} by observing the output of the mechanism over the
%% dataset.

\subsection{Related Work}

The earlier work on differential privacy was related to
functional approximations for simple data mining tasks and data
release mechanisms~\cite{DinNis03,DwoNis04,BlumDMN05,Barak07}.  Although
many of these works have connection to machine learning problems, more
recently the design and analysis of machine learning algorithms
satisfying differential privacy has been actively
studied. Kasiviswanathan, \emph{et al.}~\cite{KasiviswanathanLNRS08}
present a framework for converting a general agnostic PAC
learning algorithm to an algorithm that satisfies privacy constraints.
Chaudhuri and Monteleoni \cite{ChaudhuriM08} use the exponential
mechanism~\cite{DworkMNS06} to create a differentially private logistic
regression classifier by adding Laplace noise to the estimated
parameters. They propose another differentially
private formulation which involves modifying the objective function of
the logistic regression classifier by adding a linear term scaled by
Laplace noise. The second formulation is advantageous because it
is independent of the classifier sensitivity which
difficult to compute in general and it can be shown that using a
perturbed objective function introduces a lower error as compared to
the exponential mechanism.

However, the above mentioned differentially private classification
algorithms only address the problem of binary classification. Although
it is possible to extend binary classification algorithms to
multi-class using techniques like one-vs-all, it is much more
expensive to do so as compared to a naturally multi-class
classification algorithm. Jagannathan, \emph{et
  al.}~\cite{JagannathanPW09} present a differentially private random
decision tree learning algorithm which can be applied to multi-class
classification. Their approach involves perturbing leaf nodes using
the sensitivity method, and they do not provide theoretical analysis
of excess risk of the perturbed classifier.
In this paper, we propose a modification to the naturally multi-class
large margin Gaussian classification algorithm~\cite{ShaS06,ShaS06b}.
%% Our proposed algorithm also involves a perturbed objective function
%% that was inspired by \cite{ChaudhuriM08}, however, instead of a linear
%% term, we add a scalar perturbation to the regularization parameter
%% which leads to a much lower excess risk bound inversely proportional
%% to the data dimensionality and directly proportional to the number of
%% classes.

\section{Large Margin Gaussian Classifiers}
We investigate the large margin multi-class classification algorithm
introduced by Sha and Saul~\cite{ShaS06}. The training dataset
$(\vec{x},\vec{y})$\footnote{Notation: vectors and matrices are
  denoted by \textbf{boldface}.} contains $n$ iid $d$-dimensional training data
instances $\vec{x}_i \in \mR^d$ each with labels $y_i \in
\{1,\ldots,C\}$. We consider the setting where each class is modeled
as a single Gaussian ellipsoid. Each class ellipsoid is parametrized
by the centroid $\vec{\mu}_c \in \mR^d$, the inverse covariance matrix
$\mathbf{\Psi}_c \in \mR^{d \times d }$, and a scalar offset $\theta_c
\ge 0$. The decision rule is to assign an instance $\vec{x}_i$ to the
class having smallest Mahalanobis distance~\cite{Mahalanobis} with the
scalar offset from $\vec{x}_i$ to the centroid of that class.
\begin{align}
 y_i = \argmin{c} (\vec{x}_i - \vec{\mu}_c)^T \mathbf{\Psi}_c (\vec{x}_i - \vec{\mu}_c) + \theta_c. 
\end{align}
To simplify the notation, we expand $(\vec{x}_i - \vec{\mu}_c)^T \mathbf{\Psi}_c (\vec{x}_i
- \vec{\mu}_c)$ and collect the parameters for each class as the
following $(d+1)\times(d+1)$ \emph{positive semidefinite} matrix
\begin{align}
  \mathbf{\Phi}_c = \begin{bmatrix}
    \mathbf{\Psi}_c ~~&~~ -\mathbf{\Psi}_c\vec{\mu}_c \\
    -\vec{\mu}_c^T\mathbf{\Psi}_c ~~&~~ \vec{\mu}_c^T\mathbf{\Psi}_c\vec{\mu}_c + \theta_c
  \end{bmatrix}
\end{align}
and also append a unit element to each $d$-dimensional vector
$\vec{x}_i$. The decision rule for a data instance $\vec{x}_i$
simplifies to
\begin{align}
  y_i = \argmin{c} \vec{x}_i^T \mathbf{\Phi}_c \vec{x}_i.
\end{align}
The discriminative training procedure involves estimating a set of
positive semidefinite matrices
$\{\mathbf{\Phi}_1,\ldots,\mathbf{\Phi}_C\}$ from the training data
$\{(\vec{x}_1,y_1),\ldots,(\vec{x}_n,y_n)\}$ which optimize the
performance on the decision rule mentioned above. We apply the large
margin intuition that the optimal classifier must maximize the
distance of training data instances from the decision boundaries. This
leads to the classification algorithm being robust to outliers with
provably strong generalization guarantees. Formally, we require that
for each training data instance $\vec{x}_i$ with label $y_i$, the
distance from $\vec{x}_i$ to the centroid of class $y_i$ is at least
less than its distance from centroids of all other classes by one.
\[ \forall c \ne y_i: \vec{x}_i^T \mathbf{\Phi}_c \vec{x}_i 
 ~\ge~ 1 + \vec{x}_i^T  \mathbf{\Phi}_{y_i} \vec{x}_i. \] 
Analogous to support vector machines, the training algorithm is an
optimization problem minimizing the \emph{hinge loss} denoted by
$[f]_+ = \max(0,f)$, with a linear penalty for incorrect
classification. We use the sum of traces of inverse covariance
matrices for each classes as a \emph{regularization} term. The
regularization requires that if we can learn a classifier which labels
every training data instance correctly, we choose the one with the
lowest inverse covariance or highest covariance for each class
ellipsoid as this prevents the classifier from over-fitting. The
parameter $\lambda$ controls the trade off between the loss function
and the regularization.
\begin{align}
  J(\mathbf{\Phi},\vec{x},\vec{y}) 
  &= \sum_i \sum_{c \ne y_i} \left[1 + \vec{x}_i^T(\mathbf{\Phi}_{y_i} 
    - \mathbf{\Phi}_{c})\vec{x}_i\right]_+ + \lambda \sum_c \text{trace}(\mathbf{\Psi}_c).
\end{align}
%% As we will see in the next section, modifying the regularization
%% term plays a key role in our proposed algorithm that satisfies
%% the privacy constraints.
The inverse covariance matrix $\mathbf{\Psi}_c$ is contained in the
upper left size $d \times d$ block of the matrix $\mathbf{\Phi}_c$. We
replace it with $\mathbf{I_\Phi}\mathbf{\Phi}_c\mathbf{I_\Phi}$, where
$\mathbf{I_\Phi}$ is the truncated size $(d+1)\times(d+1)$ identity matrix
with the last diagonal element $I_{\Phi_{d+1,d+1}}$ set to zero. The
optimization problem becomes
\begin{align}
  J(\mathbf{\Phi},\vec{x},\vec{y}) 
  &= \sum_i \sum_{c \ne y_i} \left[1 + \vec{x}_i^T(\mathbf{\Phi}_{y_i} 
    - \mathbf{\Phi}_{c})\vec{x}_i\right]_+
  + \lambda \sum_c \text{trace}(\mathbf{\mathbf{I_\Phi}\Phi}_c\mathbf{I_\Phi}) \nonumber\\
  &= L(\mathbf{\Phi},\vec{x},\vec{y}) + N(\mathbf{\Phi}).
\end{align}

The hinge loss being non-differentiable is not very convenient for our
analysis; we replace it with a surrogate loss function called Huber loss
$l_h$~\cite{Chapelle07} which has similar characteristics to the hinge
loss for small values of $h$.
\begin{align}\label{eqn:huber-loss}
  \ell_h(\mathbf{\Phi}_c,x_i,y_i) = \begin{cases}
    0 & \text{if}~ \vec{x}_i^T(\mathbf{\Phi}_{c} - \mathbf{\Phi}_{y_i})\vec{x}_i > h, \\
    \frac{1}{4h}\left[h- \vec{x}_i^T(\mathbf{\Phi}_{y_i} -
      \mathbf{\Phi}_c)\vec{x}_i\right]^2 & \text{if}~ |\vec{x}_i^T(\mathbf{\Phi}_c - \mathbf{\Phi}_c)\vec{x}_i|\le h \\
    -\vec{x}_i^T(\mathbf{\Phi}_{y_i} - \mathbf{\Phi}_c)\vec{x}_i
    & \text{if}~ \vec{x}_i^T(\mathbf{\Phi}_c - \mathbf{\Phi}_{y_i})\vec{x}_i < -h. 
  \end{cases} 
\end{align}
The objective function is convex function of positive semidefinite matrices
$\mathbf{\Phi}_c$. The optimization can be formulated as a
semidefinite programming problem~\cite{VB96} and be solved efficiently
using interior point methods.

The large margin classification framework can be easily extended to
modeling each class with a mixture of Gaussians. Similar to support
vector machines, when training with non-separable data, we can
introduce slack parameters to permit margin violations. These
extensions do not change the basic characteristics of the learning
algorithm. The optimization problem remains to be a convex
semidefinite program with piecewise linear terms and is equally
tractable. For simplicity, we restrict our discussion to single
Gaussians and hard margins in this paper. As we shall see, it is easy
to extend our proposed modifications to these cases.

\section{Differentially Private Large Margin Gaussian Classifiers}
\label{sec:proposed}

We modify the large margin Gaussian classification formulation to
satisfy differential privacy by introducing a perturbation term in the
objective function. As we will see in Section
\ref{sec:error-bound}, this modification leads to a classifier that
preserves differential privacy.

We generate the size $(d+1) \times (d+1)$ perturbation matrix $\mathbf{b}$ 
with density
\begin{align}\label{eqn:b-distribution}
  P(\mathbf{b}) \propto \exp\left(-\frac{\epsilon}{2} \|\mathbf{b}\| \right),
\end{align}
where $\|\cdot\|$ is the Frobenius norm (element-wise $\ell_2$
norm) and $\epsilon$ is the privacy parameter. One method of generating
such a $\mathbf{b}$ matrix is to sample the norm $\|\mathbf{b}\|$ from 
$\Gamma\left((d+1)^2,\frac{2}{\epsilon}\right)$ and the direction of
$\mathbf{b}$ at random.

Our proposed learning algorithm minimizes the following objective
function $J_p(\mathbf{\Phi},\vec{x},\vec{y})$, where the subscript $p$
denotes privacy.
\begin{align} \label{eqn:objective}
  &J_p(\mathbf{\Phi},\vec{x},\vec{y})
  = L(\mathbf{\Phi},\vec{x},\vec{y})
  + \lambda \sum_c \text{trace}(\mathbf{I_\Phi}\mathbf{\Phi}_c\mathbf{I_\Phi})
  + \sum_c \sum_{ij} b_{ij} \Phi_{cij} \nonumber\\
  &=  J(\mathbf{\Phi},\vec{x},\vec{y}) 
  + \sum_c \sum_{ij} b_{ij} \Phi_{cij}.
\end{align}
As the dimensionality of the perturbation matrix $\mathbf{b}$ is same
as that of the classifier parameters $\mathbf{\Phi}_c$, the parameter
space of $\mathbf{\Phi}$ does not change after perturbation. In other
words, given two datasets $(\vec{x},\vec{y})$ and
$(\vec{x}',\vec{y}')$, if $\mathbf{\Phi^p}$ minimizes
$J_p(\mathbf{\Phi},\vec{x},\vec{y})$, it is always possible to have 
$\mathbf{\Phi^p}$ minimize
$J_p(\mathbf{\Phi},\vec{x}',\vec{y}')$. This is a necessary condition
for the classifier $\mathbf{\Phi^p}$ satisfying differential privacy.

Furthermore, as the perturbation term is convex and positive semidefinite, the
perturbed objective function $J_p(\mathbf{\Phi},\vec{x},\vec{y})$ has
the same properties as the unperturbed objective function
$J(\mathbf{\Phi},\vec{x},\vec{y})$. Also, the perturbation does not
introduce any additional computational cost as compared to the
original algorithm.

\section{Theoretical Analysis}

\subsection{Proof of Differential Privacy}
In the following theorem, we prove that the classifier minimizing the
perturbed optimization function $J_p(\mathbf{\Phi},\vec{x},\vec{y})$
satisfies $\epsilon$-differential privacy.  Given the dataset
$(\vec{x},\vec{y})=\{(\vec{x}_1,y_1),\ldots,(\vec{x}_{n-1},y_{n-1}),(\vec{x}_n,y_n)\}$,
the probability of learning the classifier $\mathbf{\Phi}^p$ is close
to the the probability of learning the same classifier
$\mathbf{\Phi}^p$ given its adjacent dataset
$(\vec{x}',\vec{y}')=\{(\vec{x}_1,y_1),\ldots,(\vec{x}_{n-1},y_{n-1}),(\vec{x}_n',y_n')\}$
differing wlog on the $n^{\rm th}$ instance. As we mentioned in the previous section,
it is always possible to find such a classifier $\mathbf{\Phi}^p$ 
minimizing both $J_p(\mathbf{\Phi},\vec{x},\vec{y})$ and 
$J_p(\mathbf{\Phi},\vec{x}',\vec{y}')$ due to the perturbation matrix 
being in the same space as the optimization parameters.

Our proof requires a strictly convex perturbed objective function
resulting in a unique solution $\mathbf{\Phi}^p$ minimizing it. This
in turn requires that the loss function $L(\mathbf\Phi,\vec{x},y)$ is
strictly convex and differentiable, and the regularization term
$N(\mathbf{\Phi})$ is convex. These seemingly strong constraints are
satisfied by many commonly used classification algorithms such as
logistic regression, support vector machines, and our general
perturbation technique can be extended to those algorithms. In our
proposed algorithm, the Huber loss is by definition a differentiable
function and the trace regularization term is convex and
differentiable.  Additionally, we require that the difference in the
gradients of $L(\mathbf\Phi,\vec{x},y)$ calculated over for two
adjacent training datasets is bounded. We prove this property in Lemma
\ref{lem:bounded-gradients} given in the appendix.
\begin{theorem}
  For any two adjacent training datasets $(\vec{x},\vec{y})$ and
  $(\vec{x}',\vec{y}')$, the classifier $\mathbf{\Phi}^p$ minimizing
  the perturbed objective function $J_p(\mathbf{\Phi},\vec{x},\vec{y})$ 
  satisfies differential privacy.
  \begin{align*}
    \left|\log\frac{P(\mathbf{\Phi}^p|\vec{x},\vec{y})}{P(\mathbf{\Phi}^p|\vec{x}',\vec{y}')}\right|
    \le \epsilon',
  \end{align*}
  where $\epsilon' = \epsilon + k$  for a constant factor $k = \log\left(1 +
  \frac{2\alpha}{n\lambda} + \frac{\alpha^2}{n^2\lambda^2}\right)$
  with a constant value of $\alpha$.
\end{theorem}
\begin{proof}
As $J(\mathbf{\Phi},\vec{x},\vec{y})$ is convex and differentiable, 
there is a unique solution $\mathbf{\Phi}^*$ that minimizes
it. As the perturbation term $\sum_c \sum_{ij} b_{ij}
\Phi_{cij}$ is also convex and differentiable, the perturbed objective
function $J_p(\mathbf{\Phi},\vec{x},\vec{y})$ also has a unique
solution $\mathbf{\Phi}^p$ that minimizes it.
Differentiating $J_p(\mathbf{\Phi},\vec{x},\vec{y})$ wrt
$\mathbf{\Phi}_c$, we have
\begin{align}
  &\frac{\partial}{\partial\mathbf\Phi_c} J_p(\mathbf{\Phi},\vec{x},\vec{y})
  = \frac{\partial}{\partial\mathbf\Phi_c} L(\mathbf{\Phi},\vec{x},\vec{y}) 
  + \lambda\mathbf{I_\Phi} + \mathbf{b}.
\end{align}

Substituting the optimal $\mathbf{\Phi}_c^p$ in the derivative gives us
\[ \lambda\mathbf{I_\Phi} + \mathbf{b} =
-\frac{\partial}{\partial\mathbf\Phi_c} L(\mathbf{\Phi}^p,\vec{x},\vec{y}). \]
This relation shows that two different values of $\mathbf{b}$ cannot
result in the same optimal $\mathbf{\Phi}^p$. As the perturbed
objective function $J_p(\mathbf{\Phi},\vec{x},\vec{y})$ is also convex
and differentiable, there is a bijective map between the perturbation
$\mathbf{b}$ and the unique $\mathbf{\Phi}^p$ minimizing
$J_p(\mathbf{\Phi},\vec{x},\vec{y})$. 

Let $\mathbf{b_1}$ and $\mathbf{b_2}$ be the two perturbations
applied when training with the adjacent datasets $(\vec{x},\vec{y})$
and $(\vec{x}',\vec{y}')$, respectively. Assuming that we obtain the
same optimal solution $\mathbf{\Phi}^p$ while minimizing
both $J_p(\mathbf{\Phi},\vec{x},\vec{y})$ with perturbation $\mathbf{b}_1$ and 
$J_p(\mathbf{\Phi},\vec{x},\vec{y})$ with perturbation $\mathbf{b}_2$,
\begin{align} \label{eqn:b1-b2}
  &\lambda\mathbf{I_\Phi} + \mathbf{b}_1 =
  -\frac{\partial}{\partial\mathbf\Phi_c}
    L(\mathbf{\Phi}^p,\vec{x},\vec{y}), \nonumber \\
  &\lambda\mathbf{I_\Phi} + \mathbf{b}_2 =
  -\frac{\partial}{\partial\mathbf\Phi_c}
    L(\mathbf{\Phi}^p,\vec{x}',\vec{y}'), \nonumber \\
  &\mathbf{b}_1-\mathbf{b}_2
  = \frac{\partial}{\partial\mathbf\Phi_c}
  L(\mathbf{\Phi}^p,\vec{x}',\vec{y}') -
  \frac{\partial}{\partial\mathbf\Phi_c} L(\mathbf{\Phi}^p,\vec{x},\vec{y}).  
\end{align}
We apply Lemma \ref{lem:bounded-gradients} after taking Frobenius norm
on both sides.
%% as $\mathbf{I_\Phi}$ is a truncated identity matrix
%% with $d$ diagonal entries, $\|\mathbf{I_\Phi}\|=d$.
\begin{align*} 
  &\|\mathbf{b}_1 - \mathbf{b}_2\| = \left\|\frac{\partial}{\partial\mathbf\Phi_c}
    L(\mathbf{\Phi}^p,\vec{x}',\vec{y}') -
    \frac{\partial}{\partial\mathbf\Phi_c} L(\mathbf{\Phi}^p,\vec{x},\vec{y})\right\| \\
  &= \left\|\sum_{i=1}^{n-1} \frac{\partial}{\partial\mathbf\Phi_c}
    L(\mathbf{\Phi}^p,\vec{x}_i,y_i) 
    + \frac{\partial}{\partial\mathbf\Phi_c}
    L(\mathbf{\Phi}^p,\vec{x}_n',y_n') \right. \\
  &~~~~~~\left.
    -\sum_{i=1}^{n-1} \frac{\partial}{\partial\mathbf\Phi_c} 
    L(\mathbf{\Phi}^p,\vec{x}_i,y_i) 
    - \frac{\partial}{\partial\mathbf\Phi_c}
    L(\mathbf{\Phi}^p,\vec{x}_n,y_n)\right\| \\
  &= \left\|\frac{\partial}{\partial\mathbf\Phi_c}
    L(\mathbf{\Phi}^p,\vec{x}_n',y_n') -
    \frac{\partial}{\partial\mathbf\Phi_c} L(\mathbf{\Phi}^p,\vec{x}_n,y_n)\right\|
    \le 2.
\end{align*}

Using this property, we can calculate the ratio of densities of
drawing the perturbation matrices $\mathbf{b}_1$ and $\mathbf{b}_2$ as
\begin{align*}
  \frac{P(\mathbf{b} = \mathbf{b}_1)}{P(\mathbf{b} = \mathbf{b}_2)}
  &= \frac{\frac{1}{\text{surf}(\|\mathbf{b}_1\|)}
    \|\mathbf{b}_1\|^{d}\exp\left[-\frac{\epsilon}{2}\|\mathbf{b}_1\|\right]}
  {\frac{1}{\text{surf}(\|\mathbf{b}_2\|)}
    \|\mathbf{b}_2\|^{d}\exp\left[-\frac{\epsilon}{2} \|\mathbf{b}_2\|\right]},
\end{align*}
where $\text{surf}(\|\mathbf{b}\|)$ is the surface area of the 
$(d+1)$-dimensional hypersphere with radius $\|\mathbf{b}\|$. 
As $\text{surf}(\|\mathbf{b}\|)=\text{surf}(1)\|\mathbf{b}\|^d$, where 
$\text{surf}(1)$ is the area of the unit $(d+1)$-dimensional hypersphere, 
the ratio of the densities becomes
\begin{align}\label{eqn:dp-bound}
  \frac{P(\mathbf{b} = \mathbf{b}_1)}{P(\mathbf{b} = \mathbf{b}_2)}
  &= \exp\left[\frac{\epsilon}{2} (\|\mathbf{b}_2\|-\|\mathbf{b}_1\|)\right]
  \le \exp\left[\frac{\epsilon}{2} \|\mathbf{b}_2-\mathbf{b}_1\|\right]
  \le \exp(\epsilon).
\end{align}

The ratio of the densities of learning $\mathbf{\Phi}^p$ using the adjacent datasets
$(\vec{x},\vec{y})$ and $(\vec{x}',\vec{y}')$ is given by
\begin{align}
  \frac{P(\mathbf{\Phi}^p|\vec{x},\vec{y})}{P(\mathbf{\Phi}^p|\vec{x}',\vec{y}')}
  &= \frac{P(\mathbf{b} = \mathbf{b}_1)}{P(\mathbf{b} = \mathbf{b}_2)} 
  \frac{|\det(\mathbf{J}(\mathbf{\Phi}^p \rightarrow \mathbf{b}_1|\vec{x},\vec{y}))|^{-1}}
  {|\det(\mathbf{J}(\mathbf{\Phi}^p \rightarrow \mathbf{b}_2|\vec{x}',\vec{y}'))|^{-1}},
\end{align}
where $\mathbf{J}(\mathbf{\Phi}^p \rightarrow \mathbf{b}_1|\vec{x},\vec{y})$ and 
$\mathbf{J}(\mathbf{\Phi}^p \rightarrow \mathbf{b}_2|\vec{x}',\vec{y}')$ are 
the Jacobian matrices of the bijective mappings from $\mathbf{\Phi}^p$
to $\mathbf{b}_1$ and $\mathbf{b}_2$, respectively. Following a
procedure identical to Theorem 2 of \cite{ChaudhuriMS10} (omitted due
to lack of space), it can be shown that the
ratio of Jacobian determinants is upper bounded by a constant factor 
$\exp(k) = 1 + \frac{2\alpha}{n\lambda} + \frac{\alpha^2}{n^2\lambda^2}$
for a constant value of $\alpha$. Therefore, the ratio of the
densities of learning $\mathbf{\Phi}^p$ using the adjacent datasets
becomes
\begin{align}\label{eqn:final-density-bound}
  \frac{P(\mathbf{\Phi}^p|\vec{x},\vec{y})}{P(\mathbf{\Phi}^p|\vec{x}',\vec{y}')}
  \le \exp(\epsilon+k) = \exp(\epsilon').
\end{align}

Similarly, we can show that the probability ratio is lower bounded by
$\exp(-\epsilon')$, which together with Equation \eqref{eqn:final-density-bound}
satisfies the definition of differential privacy. \\
\qed
\end{proof}

\subsection{Analysis of Excess Error}
\label{sec:error-bound}

In the remainder of this section, we denote the terms
$J(\mathbf\Phi,\mathbf{x},\mathbf{y})$ and
$L(\mathbf\Phi,\mathbf{x},\mathbf{y})$ by $J(\mathbf\Phi)$ and
$L(\mathbf\Phi)$ respectively for conciseness.
To establish a bound on excess risk of the classifier given
by the proposed algorithm minimizing the perturbed objective function,
in Lemma \ref{lemma:strong-convex} we show that the  objective
function $J(\mathbf{\Phi})$ satisfies strong convexity.
%% \begin{theorem}\label{thm:strong-convex}
%%   Given the number of classes $C$, the objective function
%%   $J(\mathbf{\Phi})$ satisfies the following 
%%   definition of strong convexity.
%%   \begin{align*}
%%   J\left(\frac{\mathbf{\Phi}+\mathbf{\Phi}'}{2}\right)
%%   &\le \frac{J(\mathbf{\Phi}) + J(\mathbf{\Phi}')}{2}
%%   - \frac{\lambda}{8dC}\left[\sum_c \text{trace}[\mathbf{I_\Phi}(\mathbf{\Phi}_c-\mathbf{\Phi}_c')\mathbf{I_\Phi}]\right]^2.
%%   \end{align*}
%% \end{theorem}
The objective function $J(\mathbf{\Phi})$ contains the loss function
$L(\mathbf{\Phi})$ computed over the training data
$(\mathbf{x},\mathbf{y})$ and the regularization term
$N(\mathbf{\Phi})$ -- this is known as the regularized \emph{empirical
  risk} of the classifier $\mathbf{\Phi}$. In the following theorem,
we establish a bound on the regularized empirical excess risk of the
differentially private classifier minimizing the perturbed objective
function over the classifier minimizing the unperturbed objective
function.
\begin{theorem}\label{thm:empirical-bound}
  With probability at least $1-\delta$,
  the regularized empirical excess risk of the classifier
  $\mathbf{\Phi}^p$ minimizing the perturbed objective function 
  $J_p(\mathbf{\Phi})$ over the classifier $\mathbf{\Phi}^*$
  minimizing the unperturbed objective function $J(\mathbf{\Phi})$
  is bounded as
  \begin{align*}
    J(\mathbf{\Phi}^p) \le J(\mathbf{\Phi}^*) 
    + \frac{8(d+1)^4C}{\epsilon^2\lambda}\log^2\left(\frac{d}{\delta}\right).
  \end{align*}
\end{theorem}
\begin{proof}
  We use the definition of
  $J_p(\mathbf{\Phi}) = J(\mathbf{\Phi}) 
  + \sum_c \sum_{ij} b_{ij} \Phi_{cij}$ and
  the optimality of $\mathbf{\Phi}^p$, \emph{i.e.}, 
  $J_p(\mathbf{\Phi}^p) \le J_p(\mathbf{\Phi}^*)$.
  \begin{align} \label{eqn:base-bound1}
    &J(\mathbf{\Phi}^p) + \sum_c \sum_{ij} b_{ij} \Phi_{cij}^p
    \le J(\mathbf{\Phi}^*) + \sum_c \sum_{ij} b_{ij} \Phi_{cij}^*, \nonumber \\
    &J(\mathbf{\Phi}^p) \le J(\mathbf{\Phi}^*)
    + \sum_c \sum_{ij} b_{ij} (\Phi_{cij}^*-\Phi_{cij}^p).
  \end{align}
  Using the strong convexity of $J(\mathbf{\Phi})$ as given by Lemma
  \ref{lemma:strong-convex} and the optimality of
  $J(\mathbf{\Phi}^*)$, we have 
  \begin{align}\label{eqn:j-relation}
    J(\mathbf{\Phi}^*) \le J\left(\frac{\mathbf{\Phi}^p+\mathbf{\Phi}^*}{2}\right)
    &\le \frac{J(\mathbf{\Phi}^p) + J(\mathbf{\Phi}^*)}{2}
    - \frac{\lambda}{8}\sum_c \|\mathbf{\Phi}_c^* - \mathbf{\Phi}_c^p\|^2, \nonumber \\
    J(\mathbf{\Phi}^p) - J(\mathbf{\Phi}^*)
    &\ge \frac{\lambda}{4}\sum_c \|\mathbf{\Phi}_c^* - \mathbf{\Phi}_c^p\|^2.
  \end{align}
  Similarly, using the strong convexity of $J_p(\mathbf{\Phi})$ and
  the optimality of $J_p(\mathbf{\Phi}^p)$,
  \begin{align*}
    J_p(\mathbf{\Phi}^p) \le J_p\left(\frac{\mathbf{\Phi}^p+\mathbf{\Phi}^*}{2}\right)
    &\le \frac{J_p(\mathbf{\Phi}^p) + J_p(\mathbf{\Phi}^*)}{2}
    - \frac{\lambda}{8}\sum_c \|\mathbf{\Phi}_c^p - \mathbf{\Phi}_c^*\|^2, \\
    J_p(\mathbf{\Phi}^*) - J_p(\mathbf{\Phi}^p)
    &\ge \frac{\lambda}{4} \sum_c \|\mathbf{\Phi}_c^p - \mathbf{\Phi}_c^*\|^2.
  \end{align*}
  Substituting the definition $J_p(\mathbf{\Phi}) = J(\mathbf{\Phi})
  + \sum_c \sum_{ij} b_{ij}\Phi_{cij}$,
  \begin{align*}
   &J(\mathbf{\Phi}^*) + \sum_c \sum_{ij} b_{ij}\Phi_{cij}^*
    - J(\mathbf{\Phi}^p) - \sum_c \sum_{ij} b_{ij}\Phi_{cij}^p 
   \ge \frac{\lambda}{4} \sum_c \|\mathbf{\Phi}_c^* - \mathbf{\Phi}_c^p\|^2 \\
   &\sum_c \sum_{ij} b_{ij}(\Phi_{cij}^* - \Phi_{cij}^p)
    - (J(\mathbf{\Phi}^p) - J(\mathbf{\Phi}^*)) 
   \ge \frac{\lambda}{4} \sum_c \|\mathbf{\Phi}_c^* - \mathbf{\Phi}_c^p\|^2.
  \end{align*}
  Substituting the lower bound on $J(\mathbf{\Phi}^p) - J(\mathbf{\Phi}^*)$
  given by Equation \eqref{eqn:j-relation},
  \begin{align}\label{eqn:phi-sum}
   \sum_c \sum_{ij} b_{ij}(\Phi_{cij}^* - \Phi_{cij}^p)
    &\ge \frac{\lambda}{2} \sum_c \|\mathbf{\Phi}_c^* - \mathbf{\Phi}_c^p\|^2, \nonumber \\
   \left[\sum_c \sum_{ij} b_{ij}(\Phi_{cij}^* - \Phi_{cij}^p)\right]^2
   &\ge \frac{\lambda^2}{4} \left[\sum_c \|\mathbf{\Phi}_c^* - \mathbf{\Phi}_c^p\|^2\right]^2.
  \end{align}
Using the Cauchy-Schwarz inequality, we have,
\begin{align}\label{eqn:cauchy-schwarz}
  \left[\sum_c \sum_{ij} b_{ij}(\Phi_{cij}^* - \Phi_{cij}^p)\right]^2
  \le C\|\mathbf{b}\|^2 \sum_c \|\mathbf{\Phi}_c^* - \mathbf{\Phi}_c^p\|^2    
\end{align}
Combining this with Equation \eqref{eqn:phi-sum} gives us
\begin{align}\label{eqn:sum-norm-bound}
  C\|\mathbf{b}\|^2 \sum_c \|\mathbf{\Phi}_c^* - \mathbf{\Phi}_c^p\|^2
  &\ge \frac{\lambda^2}{4} \left[\sum_c
    \|\mathbf{\Phi}_c^*-\mathbf{\Phi}_c^p\|^2\right]^2, \nonumber \\
  \sum_c \|\mathbf{\Phi}_c^*-\mathbf{\Phi}_c^p\|^2 
  &\le \frac{4C}{\lambda^2}  \|\mathbf{b}\|^2.
\end{align}
Combining this with Equation \eqref{eqn:cauchy-schwarz} gives us
\begin{align*}
  \sum_c \sum_{ij} b_{ij}(\Phi_{cij}^* - \Phi_{cij}^p)
  \le \frac{2C}{\lambda} \|\mathbf{b}\|^2.
\end{align*}
We bound $\|\mathbf{b}\|^2$ with probability at least
$1-\delta$ as given by Lemma \ref{lem:exp-bound}.
\begin{align}
  \sum_c \sum_{ij} b_{ij}(\Phi_{cij}^* - \Phi_{cij}^p)
  \le \frac{8(d+1)^4C}{\epsilon^2\lambda}\log^2\left(\frac{d}{\delta}\right).
\end{align}
Substituting this in Equation \eqref{eqn:base-bound1} proves the
theorem. \\
\qed
\end{proof}

The upper bound on the regularized empirical risk is in
$O(\frac{C}{\epsilon^2})$. The bound increases for smaller values of
$\epsilon$ which implies tighter privacy and therefore suggests a
trade off between privacy and utility.

The regularized empirical risk of a classifier is calculated over a
given training dataset. In practice, we are more interested in how the
classifier will perform on new test data which is assumed to be
generated from the same source as the training data. The expected
value of the loss function computed over the data is called the
\emph{true risk} $\tilde{L}(\mathbf{\Phi}) = \mE [L(\mathbf{\Phi})]$
of the classifier $\mathbf{\Phi}$. In the following theorem, we
establish a bound on the true excess risk of the differentially
private classifier minimizing the perturbed objective function and the
classifier minimizing the original objective function.
\begin{theorem}
  With probability at least $1-\delta$, the true excess risk of the
  classifier $\mathbf{\Phi}^p$ minimizing the perturbed objective
  function $J_p(\mathbf{\Phi})$ over the classifier
  $\mathbf{\Phi}^*$ minimizing the unperturbed objective function
  $J(\mathbf{\Phi})$ is bounded as
  \begin{align*}
    \tilde{L}(\mathbf{\Phi}^p) \le \tilde{L}(\mathbf{\Phi}^*)
    &+ \frac{4 \sqrt{d}(d+1)^2C}{\epsilon\lambda}\log\left(\frac{d}{\delta}\right) \\
    &+ \frac{8(d+1)^4C}{\epsilon^2\lambda}\log^2\left(\frac{d}{\delta}\right)
    + \frac{16}{\lambda n} \left[32 + \log\left(\frac{1}{\delta}\right)\right].
  \end{align*}
\end{theorem}
\begin{proof}
  Let the expected value of the regularized empirical risk be
  \begin{align}
    \tilde{J}(\mathbf{\Phi}) = \tilde{L}(\mathbf{\Phi})
    + \lambda \sum_c \text{trace}(\mathbf{I_\Phi}\mathbf{\Phi}_c\mathbf{I_\Phi}).
  \end{align}
  Let $\mathbf{\Phi}^r$ be the classifier minimizing
  $\tilde{J}(\mathbf{\Phi})$, \emph{i.e.}, 
  $\tilde{J}(\mathbf{\Phi}^r) \le \tilde{J}(\mathbf{\Phi}^*)$. \\
  Rearranging the terms, we have
  \begin{align*}
    \tilde{J}(\mathbf{\Phi}^p) &= \tilde{J}(\mathbf{\Phi}^*)
    + [\tilde{J}(\mathbf{\Phi}^p) - \tilde{J}(\mathbf{\Phi}^r)]
    + [\tilde{J}(\mathbf{\Phi}^r) - \tilde{J}(\mathbf{\Phi}^*)] \\
    &\le \tilde{J}(\mathbf{\Phi}^*) + [\tilde{J}(\mathbf{\Phi}^p) - \tilde{J}(\mathbf{\Phi}^r)].
  \end{align*}
  Substituting the definition of $\tilde{J}(\mathbf{\Phi})$,
  \begin{align}\label{eqn:base-bound2}
    \tilde{L}(\mathbf{\Phi}^p) &+ \lambda \sum_c \text{trace}(\mathbf{I_\Phi}\mathbf{\Phi}_c^p\mathbf{I_\Phi})
    \le \tilde{L}(\mathbf{\Phi}^*) + \lambda \sum_c \text{trace}(\mathbf{I_\Phi}\mathbf{\Phi}_c^*\mathbf{I_\Phi})
    + [\tilde{J}(\mathbf{\Phi}^p) - \tilde{J}(\mathbf{\Phi}^r)] \nonumber\\
    \tilde{L}(\mathbf{\Phi}^p)
    &\le \tilde{L}(\mathbf{\Phi}^*)
    + \lambda \sum_c \text{trace}[\mathbf{I_\Phi}(\mathbf{\Phi}_c^*-\mathbf{\Phi}_c^p)\mathbf{I_\Phi}]
    + [\tilde{J}(\mathbf{\Phi}^p) - \tilde{J}(\mathbf{\Phi}^r)].
  \end{align}

  From Lemma \ref{lemma:trace-norm} and Equation \eqref{eqn:sum-norm-bound},
  we have,
  \begin{align*}
    \left[\sum_c \text{trace}[\mathbf{I_\Phi}(\mathbf{\Phi}_c^*-\mathbf{\Phi}_c^p)\mathbf{I_\Phi}]\right]^2 
    &\le dC \sum_c \left\|\mathbf{\Phi}_c-\mathbf{\Phi}_c'\right\|^2 \nonumber \\
    &\le \frac{4dC^2}{\lambda^2}  \|\mathbf{b}\|^2
    = \frac{16 d(d+1)^4C^2}{\epsilon^2\lambda^2}\log^2\left(\frac{d}{\delta}\right).
  \end{align*}
  Taking the square root,
  \begin{align}\label{eqn:lambda-phi-sum}
    \sum_c \text{trace}[\mathbf{I_\Phi}(\mathbf{\Phi}_c^*-\mathbf{\Phi}_c^p)\mathbf{I_\Phi}]
    \le \frac{4 \sqrt{d}(d+1)^2C}{\epsilon\lambda}\log\left(\frac{d}{\delta}\right).
  \end{align}  

  Sridharan, \emph{et al.}~\cite{SridharanSS08} present a bound on the
  true excess risk of any classifier as an expression of the bound on the
  regularized empirical excess risk for that classifier. With probability at least
  $1-\delta$,
  \begin{align*}
    \tilde{J}(\mathbf{\Phi}^p) - \tilde{J}(\mathbf{\Phi}^r)
    \le 2 [J(\mathbf{\Phi}^p) - J(\mathbf{\Phi}^*)]
  + \frac{16}{\lambda n} \left[32 + \log\left(\frac{1}{\delta}\right)\right].
  \end{align*}
  Substituting the bound from Theorem \ref{thm:empirical-bound},
  \begin{align}\label{eqn:bound-true-risk}
    \tilde{J}(\mathbf{\Phi}^p) - \tilde{J}(\mathbf{\Phi}^r)
    \le \frac{8(d+1)^4C}{\epsilon^2\lambda}\log^2\left(\frac{d}{\delta}\right)
    + \frac{16}{\lambda n} \left[32 + \log\left(\frac{1}{\delta}\right)\right].
  \end{align}
  Substituting the results from Equations \eqref{eqn:lambda-phi-sum} and
  \eqref{eqn:bound-true-risk} into Equation \eqref{eqn:base-bound2} proves
  the theorem. \\
  \qed
\end{proof}

Similar to the bound on the regularized empirical excess risk, the
bound on the true excess risk is also inversely proportional to $\epsilon$
reflecting the privacy-utility trade-off. The bound is linear in the
number of classes $C$, which is a consequence of the multi-class
classification.  The classifier learned using a higher value of the
regularization parameter $\lambda$ will have a higher covariance for
each class ellipsoid. This would also make the classifier less
sensitive to the perturbation. This intuition is confirmed by the fact
that the true excess risk bound is inversely proportional to
$\lambda$.%%  Finally, the bound on true excess risk for the proposed
%% algorithm is directly proportional to the dimensionality of the
%% training data $d$

\section{Conclusion}
In this paper, we present a discriminatively trained Gaussian
classification algorithm that satisfies differential privacy. Our
proposed technique involves adding a perturbation term to
the objective function. We prove that the proposed algorithm
satisfies differential privacy and establish a bound on the excess
risk of the classifier learned by the algorithm which is inversely
proportional to the data dimensionality which is directly proportional to
the number of classes and inversely proportional to the privacy
parameter $\epsilon$ reflecting a trade-off between privacy and utility.

In the future, we plan to extend this work along two main directions:
extending our perturbation technique for a general class of
learning algorithms and applying results from theory of large margin
classifiers to arrive at tighter excess risk bounds for the
differentially private large margin classifiers. Our intuition is that
compared to other classification algorithms, a large margin classifier
should be much more robust to perturbation. This would also give us
insights into designing low error inducing mechanisms for
differentially private classifiers.

\section*{Acknowledgements}
We would like to thank the anonymous reviewers for their insightful
comments.

\bibliographystyle{splncs}
\bibliography{dp-lmgmm}

\begin{thebibliography}{10}

\bibitem{Dwork06}
Dwork, C.:
\newblock Differential privacy.
\newblock In: International Colloquium on Automata, Languages and Programming.
  (2006)

\bibitem{ChaudhuriM08}
Chaudhuri, K., Monteleoni, C.:
\newblock Privacy-preserving logistic regression.
\newblock In: Neural Information Processing Systems. (2008)  289--296

\bibitem{mclachlan2000}
McLachlan, G., Peel, D.:
\newblock Finite Mixture Models.
\newblock Wiley series in probability and statistics. Wiley-Interscience (2000)

\bibitem{ShaS06}
Sha, F., Saul, L.K.:
\newblock Large margin gaussian mixture modeling for phonetic classification
  and recognition.
\newblock In: IEEE International Conference on Acoustics, Speech and Signal
  Processing. (2006)  265--268

\bibitem{KasiviswanathanLNRS08}
Kasiviswanathan, S.P., Lee, H.K., Nissim, K., Raskhodnikova, S., Smith, A.:
\newblock What can we learn privately?
\newblock In: IEEE Symposium on Foundations of Computer Science. (2008)
  531--540

\bibitem{DinNis03}
Dinur, I., Nissim, K.:
\newblock Revealing information while preserving privacy.
\newblock In: Symposium on Principles of Database Systems. (2003)

\bibitem{DwoNis04}
Dwork, C., Nissim, K.:
\newblock Privacy-preserving datamining on vertically partitioned databases.
\newblock In: CRYPTO. (2004)

\bibitem{BlumDMN05}
Blum, A., Dwork, C., McSherry, F., Nissim, K.:
\newblock Practical privacy: The su{LQ} framework.
\newblock In: Symposium on Principles of Database Systems. (2005)

\bibitem{Barak07}
Barak, B., Chaudhuri, K., Dwork, C., Kale, S., McSherry, F., Talwar, K.:
\newblock Privacy, accuracy, and consistency too: a holistic solution to
  contingency table release.
\newblock In: Symposium on Principles of Database Systems. (2007)  273--282

\bibitem{DworkMNS06}
Dwork, C., McSherry, F., Nissim, K., Smith, A.:
\newblock Calibrating noise to sensitivity in private data analysis.
\newblock In: Theory of Cryptography Conference. Volume 3876. (2006)  265--284

\bibitem{JagannathanPW09}
Jagannathan, G., Pillaipakkamnatt, K., Wright, R.N.:
\newblock A practical differentially private random decision tree classifier.
\newblock In: ICDM Workshop on Privacy Aspects of Data Mining. (2009)  114--121

\bibitem{ShaS06b}
Sha, F., Saul, L.K.:
\newblock Large margin hidden markov models for automatic speech recognition.
\newblock In: Neural Information Processing Systems. (2007)  1249--1256

\bibitem{Mahalanobis}
Mahalanobis, P.C.:
\newblock On the generalised distance in statistics.
\newblock Proceedings of the National Institute of Sciences of India \textbf{2}
  (1936)  49--55

\bibitem{Chapelle07}
Chapelle, O.:
\newblock Training a support vector machine in the primal.
\newblock Neural Computation \textbf{19}(5) (2007)  1155--1178

\bibitem{VB96}
Vandenberghe, L., Boyd, S.:
\newblock Semidefinite programming.
\newblock SIAM Review \textbf{38} (1996)  49--95

\bibitem{ChaudhuriMS10}
Chaudhuri, K., Monteleoni, C., Sarwate, A.D.:
\newblock Differentially private empirical risk minimization.
\newblock arXiv:0912.0071v4 [cs.LG] (2010)

\bibitem{SridharanSS08}
Sridharan, K., Shalev-Shwartz, S., Srebro, N.:
\newblock Fast rates for regularized objectives.
\newblock In: Neural Information Processing Systems. (2008)  1545--1552

\end{thebibliography}

\section*{Appendix}
\begin{lemma}\label{lem:bounded-gradients}
  Assuming all the data instances to lie within a unit $\ell_2$ ball,
  the difference in the derivative of Huber loss function
  $L(\mathbf\Phi,\vec{x},y)$ calculated over two data instances
  $(\vec{x}_i,y_i)$ and $(\vec{x}_i',y_i')$ is bounded.
  \begin{align*}
    \left\| \frac{\partial}{\partial\mathbf{\Phi}_c} L(\mathbf{\Phi},\vec{x}_i,y_i) - 
     \frac{\partial}{\partial\mathbf\Phi_c} L(\mathbf\Phi,\vec{x}_i',y_i') \right\| ~\le~ 2.
  \end{align*}
\end{lemma}
\begin{proof}
  The derivative of the Huber loss function for the data instance
  $\vec{x}_i$ with label $y_i$ is
  \begin{align*}
    \frac{\partial}{\partial\mathbf\Phi_c} L(\mathbf\Phi,\vec{x}_i,y_i)
    = \begin{cases}
      0 & ~if~ \vec{x}_i^T(\mathbf{\Phi}_c - \mathbf{\Phi}_{y_i})\vec{x}_i > h, \\
      \frac{1}{2h}[h - \vec{x}_i^T(\mathbf{\Phi}_{y_i} -
        \mathbf{\Phi}_c)\vec{x}_i] \vec{x}_i \vec{x}_i^T  
      & ~if~ |\vec{x}_i^T(\mathbf{\Phi}_c - \mathbf{\Phi}_{y_i})\vec{x}_i|\le h, \\
      \vec{x}_i \vec{x}_i^T & ~if~ \vec{x}_i^T(\mathbf{\Phi}_c - \mathbf{\Phi}_{y_i})\vec{x}_i < -h.
    \end{cases} 
  \end{align*} 
  The data points lie in a $\ell_2$ ball of radius 1, $\forall i:
  \|\vec{x}_i\|_2 \le 1$.  Using linear algebra, it is easy to show
  that the Frobenius norm of the matrix $\vec{x}_i\vec{x}_i^T$ is same
  as the $\ell_2$ norm of the vector $\vec{x}_i$,  
  $\|\vec{x}_i\vec{x}_i^T\| = \|\vec{x}_i\|_2 \le 1$.

  As the term $\frac{1}{2h}[h - \vec{x}_i^T(\mathbf{\Phi}_{y_i}
    - \mathbf{\Phi}_c)\vec{x}_i]$ is at most one when
  $|\vec{x}_i^T(\mathbf{\Phi}_c - \mathbf{\Phi}_{y_i})\vec{x}_i|\le
  h$, the Frobenius norm of the derivative of the Huber loss function
  is at most one in all cases,
  $\left\|\frac{\partial}{\partial\mathbf\Phi_c} L(\mathbf\Phi,\vec{x}_i,y_i)\right\| 
    \le 1.$ Using a similar argument for data instance $\vec{x}_i'$ with label
  $y_i'$, we have
%  \begin{align}
    $\left\|\frac{\partial}{\partial\mathbf\Phi_c} L(\mathbf\Phi,\vec{x}_i',y_i')\right\|
    \le 1.$
%  \end{align}

  Finally, using the triangle inequality
  $\|\vec{a} - \vec{b}\| = \|\vec{a} + (-\vec{b})\| \le \|\vec{a}\| + \|\vec{b}\|$,
  \begin{align*}
    &\left\| \frac{\partial}{\partial\mathbf{\Phi}_c} L(\mathbf{\Phi},\vec{x}_i,y_i) - 
     \frac{\partial}{\partial\mathbf\Phi_c} L(\mathbf\Phi,\vec{x}_i',y_i') \right\| \\
    &\le \left\|\frac{\partial}{\partial\mathbf{\Phi}_c} L(\mathbf{\Phi},\vec{x}_i,y_i)\right\| 
    +  \left\|\frac{\partial}{\partial\mathbf\Phi_c} L(\mathbf\Phi,\vec{x}_i',y_i')\right\|
   \le 2.
  \end{align*}  
  \qed
\end{proof}

\begin{lemma}\label{lemma:strong-convex}
  The objective function $J(\mathbf{\Phi})$ is $\lambda$-strongly
  convex. For $0 \le \alpha \le 1$,
  \begin{align*}
  J\left(\alpha\mathbf{\Phi}+(1-\alpha)\mathbf{\Phi}'\right)
  \le \alpha J(\mathbf{\Phi}) + (1-\alpha)J(\mathbf{\Phi}')
  - \frac{\lambda \alpha (1-\alpha)}{2}\sum_c
  \left\|\mathbf{\Phi}_c-\mathbf{\Phi}_c'\right\|^2.
  \end{align*}
\end{lemma}
\begin{proof}
  By definition, Huber loss is $\lambda$-strongly convex,
  {\em i.e.}
  \begin{align}
    L\left(\alpha\mathbf{\Phi}+(1-\alpha)\mathbf{\Phi}'\right)
    \le \alpha L(\mathbf{\Phi}) + (1-\alpha)L(\mathbf{\Phi}')
    - \frac{\lambda \alpha
      (1-\alpha)}{2}\left\|\mathbf{\Phi}-\mathbf{\Phi}'\right\|^2.
  \end{align}
  where the Frobenius norm of the matrix set $\mathbf{\Phi}-\mathbf{\Phi}'$ is the
  sum of norms of the component matrices $\mathbf{\Phi}_c-\mathbf{\Phi}_c'$,
  \begin{align} \left\|\mathbf{\Phi}-\mathbf{\Phi}'\right\|^2
  = \sum_c \left\|\mathbf{\Phi}_c-\mathbf{\Phi}_c'\right\|^2.
  \end{align}
  As the regularization term $N(\mathbf\Phi)$ is linear,
  \begin{align}
    N(\alpha\mathbf{\Phi}+(1-\alpha)\mathbf{\Phi}')
    &= \lambda \sum_c \text{trace}(\alpha\mathbf{I_\Phi}\mathbf{\Phi}_c\mathbf{I_\Phi}
    + (1-\alpha)\mathbf{I_\Phi}\mathbf{\Phi}_c'\mathbf{I_\Phi}) \\ \nonumber
    &= \alpha\lambda \sum_c \text{trace}(\mathbf{I_\Phi}\mathbf{\Phi}_c\mathbf{I_\Phi})
    + (1-\alpha)\lambda \sum_c \text{trace}(\mathbf{I_\Phi}\mathbf{\Phi}_c'\mathbf{I_\Phi})
    \\ \nonumber
    &= \alpha N(\mathbf{\Phi})+(1-\alpha)N(\mathbf{\Phi}').
  \end{align}
  The lemma follows directly from the definition
  $J(\mathbf\Phi) = L(\mathbf\Phi) + N(\mathbf\Phi)$. \\
  \qed
\end{proof}

\begin{lemma}\label{lemma:trace-norm}
  \begin{align*}
  \frac{1}{dC}\left[\sum_c \text{trace}[\mathbf{I}_\mathbf{\Phi}(\mathbf{\Phi}_c-\mathbf{\Phi}_c')\mathbf {I}_\mathbf\Phi]\right]^2
  \le \sum_c \left\|\mathbf{\Phi}_c-\mathbf{\Phi}_c'\right\|^2
  \end{align*}
\end{lemma}
\begin{proof}
Let $\Phi_{c,i,j}$ be the $(i,j)^{\rm th}$ element of the size
$(d+1)\times(d+1)$ matrix $\mathbf{\Phi}_c-\mathbf{\Phi}_c'$. By the
definition of the Frobenius norm, and using the identity
$N\sum_{i=1}^N x_i^2 \ge (\sum_{i=1}^N x_i)^2$,
\begin{align*}
  \sum_c \left\|\mathbf{\Phi}_c-\mathbf{\Phi}_c'\right\|^2
  &= \sum_c\sum_{i=1}^{d+1}\sum_{j=1}^{d+1} \Phi_{c,i,j}^2
  \ge \sum_c\sum_{i=1}^{d+1} \Phi_{c,i,i}^2 \ge \sum_c\sum_{i=1}^{d} \Phi_{c,i,i}^2\\
  &\ge \frac{1}{dC} \left(\sum_c\sum_{i=1}^{d} \Phi_{c,i,i}\right)^2
  = \frac{1}{dC} \left[\sum_{c} \text{trace}[\mathbf{I_\Phi}(\mathbf{\Phi}_c-\mathbf{\Phi}_c')\mathbf{I_\Phi}]\right]^2.
\end{align*}
\qed
\end{proof}

\begin{lemma} \label{lem:exp-bound}
  \[ P\left[\|\mathbf{b}\| \ge 
    \frac{2 (d+1)^2}{\epsilon}\log\left(\frac{d}{\delta}\right)\right] \le \delta. \]
\end{lemma}
\begin{proof}
  Similar to the union bound argument used in Lemma 5 in \cite{ChaudhuriM08}.
\end{proof}

\end{document}